\documentclass{article}

\usepackage{microtype}
\usepackage{graphicx}
\usepackage{subcaption}
\usepackage{booktabs} 

\usepackage{hyperref}

\usepackage[preprint]{icml2026}

\usepackage[utf8]{inputenc}
\usepackage[T1]{fontenc}
\usepackage{amsfonts}

\usepackage{amsmath}
\usepackage{amssymb}
\usepackage{mathtools}
\usepackage{amsthm}
\usepackage{nicefrac}

\usepackage{algorithm}
\usepackage{algorithmic}

\usepackage[capitalize,noabbrev]{cleveref}

\usepackage{xcolor}
\usepackage{multirow}
\usepackage{url}

\theoremstyle{plain}
\newtheorem{theorem}{Theorem}[section]
\newtheorem{proposition}[theorem]{Proposition}
\newtheorem{lemma}[theorem]{Lemma}

\theoremstyle{definition}
\newtheorem{definition}[theorem]{Definition}

\theoremstyle{remark}
\newtheorem{remark}[theorem]{Remark}

\newcommand{\mhcgnn}{\textit{m}HC-GNN}
\newcommand{\mhc}{\textit{m}HC}
\newcommand{\RR}{\mathbb{R}}

\newcommand{\EE}{\mathbb{E}}

\newcommand{\hpre}[1]{\mathcal{H}^{\mathrm{pre}}_{#1}}
\newcommand{\hpost}[1]{\mathcal{H}^{\mathrm{post}}_{#1}}
\newcommand{\hres}[1]{\mathcal{H}^{\mathrm{res}}_{#1}}
\newcommand{\calF}{\mathcal{F}}
\newcommand{\calN}{\mathcal{N}}
\newcommand{\calW}{\mathcal{W}}

\newcommand{\calB}{\mathcal{B}}
\newcommand{\calV}{\mathcal{V}}
\newcommand{\calE}{\mathcal{E}}

\icmltitlerunning{\textit{m}HC-GNN: Manifold-Constrained Hyper-Connections for Graph Neural Networks}

\begin{document}

\twocolumn[
\icmltitle{\textit{m}HC-GNN: Manifold-Constrained Hyper-Connections for \\ Graph Neural Networks}

\icmlsetsymbol{equal}{*}

\begin{icmlauthorlist}
\icmlauthor{Subhankar Mishra}{inst1}
\end{icmlauthorlist}

\icmlaffiliation{inst1}{National Institute of Science Education and Research, Bhubaneswar, India}

\icmlcorrespondingauthor{Subhankar Mishra}{smishra@niser.ac.in}

\icmlkeywords{Graph Neural Networks, Deep Learning, Manifold Constraints, Over-smoothing, Expressiveness, Machine Learning, ICML}

\vskip 0.3in
]

\printAffiliationsAndNotice{}

\begin{abstract}
Graph Neural Networks (GNNs) suffer from over-smoothing in deep architectures and expressiveness bounded by the 1-Weisfeiler-Leman (1-WL) test.
We adapt Manifold-Constrained Hyper-Connections, recently proposed for Transformers, to graph neural networks.
Our method, \mhcgnn{}, expands node representations across $n$ parallel streams and constrains stream-mixing matrices to the Birkhoff polytope of doubly stochastic matrices via Sinkhorn-Knopp normalization.
We prove that \mhcgnn{} mitigates over-smoothing via a layer-wise residual lower bound showing that node-pair differences decay at rate $(1{-}\varepsilon)^L$ (where $\varepsilon$ measures deviation of the mixing matrix from identity), far slower than the standard $(1{-}\gamma)^L$ collapse rate driven by the spectral gap $\gamma$.
This two-regime analysis, via the protected orthogonal subspace for $L < n$ and the layer-wise contraction for $L \geq n$, provides architecture-agnostic rate guarantees absent from prior methods.
With independent random stream initialization, \mhcgnn{} can distinguish graphs beyond 1-WL by maintaining stream diversity across layers via doubly stochastic mixing.
Depth experiments spanning 2 to 128 layers reveal that standard GNNs collapse to near-random performance beyond 16 layers, while \mhcgnn{} maintains over 74\% accuracy at 128 layers, with improvements exceeding 50 percentage points at extreme depths.
Ablations confirm that manifold constraints are essential: removing them causes up to 82\% performance degradation.
Experiments on heterophilic graphs (roman-empire, penn94, genius) and expressiveness benchmarks (EXP) further validate the contribution. Code is available at \url{https://github.com/smishra-lab/mhc-gnn}
\end{abstract}

\section{Introduction}
\label{sec:introduction}

Graph Neural Networks (GNNs) have become a standard approach for learning on graph-structured data, with applications in social network analysis~\citep{hamilton2017inductive}, molecular property prediction~\citep{gilmer2017neural}, recommendation systems~\citep{ying2018graph}, and knowledge graph reasoning~\citep{schlichtkrull2018modeling}.
GNNs aggregate information from local neighborhoods through iterative message passing to learn node and graph-level representations~\citep{joshi2021learning}.
Despite their success, GNNs face two well-known limitations.

\textbf{Over-smoothing.}
As the number of layers increases, node representations in GNNs tend to converge to indistinguishable values, losing discriminative power necessary for downstream tasks~\citep{li2018deeper, oono2020graph}.
This prevents GNNs from benefiting from depth, limiting their ability to capture long-range dependencies.
Various techniques have been proposed: residual connections~\citep{he2016deep}, normalization schemes~\citep{zhao2019pairnorm, zhou2020towards}, and dropout strategies~\citep{rong2019dropedge}.
GCNII~\citep{chen2020simple} provides the strongest prior guarantee via initial residual connections and identity mapping, but that guarantee is GCN-specific and does not transfer to SAGE, GAT, or GIN.
PairNorm~\citep{zhao2019pairnorm} and DiffGroupNorm~\citep{zhou2020towards} are normalization heuristics without explicit convergence rate bounds.
Our contribution is an \emph{architecture-agnostic} explicit rate guarantee: the \mhcgnn{} residual path maintains node-pair differences at rate $(1-\varepsilon)^L$, controllable by the Birkhoff constraint, for any backbone GNN.

\textbf{Limited Expressiveness.}
Standard Message Passing Neural Networks (MPNNs) are bounded in their discriminative power by the 1-WL graph isomorphism test~\citep{xu2018powerful, morris2019weisfeiler}.
Recent approaches address this through higher-order methods~\citep{morris2019weisfeiler, maron2019provably}, subgraph-based techniques~\citep{bouritsas2022improving}, or graph transformers~\citep{dwivedi2020generalization}, often at the cost of increased computational complexity.

In parallel, Hyper-Connections (HC)~\citep{zhu2024hyper} introduced learnable matrices to modulate connection strengths across expanded residual streams in Transformers.
\citet{xie2025mhc} proposed Manifold-Constrained Hyper-Connections (\mhc) for language models, addressing instability in HC by projecting connection matrices onto the Birkhoff polytope via Sinkhorn-Knopp normalization~\citep{sinkhorn1967concerning}.
This ensures feature mean conservation and bounded signal propagation for stable training at scale.

\subsection{Contributions}

We propose \textbf{\mhcgnn{}}, adapting the \mhc{} framework~\citep{xie2025mhc} to graph-structured data.
Our contributions are:

\begin{itemize}
    \item \textbf{Architecture:} A backbone-agnostic framework for multi-stream GNNs where each node maintains $n$ parallel feature streams mixed through doubly stochastic matrices, with a novel design for interfacing with sparse graph adjacency (\Cref{sec:methodology}).

    \item \textbf{Theoretical Analysis:} We provide a two-regime over-smoothing analysis: a protected orthogonal subspace argument for $L < n$, and a layer-wise residual contraction (rate $(1-\varepsilon)^L$) valid for all $L \geq 1$ (\Cref{thm:oversmoothing}). We also characterize expressiveness beyond 1-WL as arising from maintained stream diversity rather than a position in the $k$-WL hierarchy (\Cref{thm:expressiveness}).

    \item \textbf{Empirical Validation:} Depth analysis from 2 to 128 layers shows \mhcgnn{} enables networks exceeding 100 layers; new experiments on heterophilic graphs (roman-empire 22K, penn94 41K, genius 421K), Amazon-Computers/Photo, and the EXP expressiveness benchmark further validate the claims (\Cref{sec:experiments}).
\end{itemize}

\section{Related Work}
\label{sec:related}

\subsection{Deep GNNs and Over-smoothing Mitigation}

\textbf{Prior methods and their guarantees.}
DropEdge~\citep{rong2019dropedge} randomly removes edges during training to slow feature diffusion, but provides no explicit rate bound.
PairNorm~\citep{zhao2019pairnorm} and DiffGroupNorm~\citep{zhou2020towards} normalize features to prevent collapse; both are heuristics without proven convergence rates.
GCNII~\citep{chen2020simple} provides the strongest guarantee: initial residual connections with identity mapping yield stability at 64 layers.
However, GCNII's guarantee relies on a specific weight decay schedule ($\beta \log(1 + l/\alpha)$) tied to GCN's normalized aggregation; it does not transfer to GraphSAGE, GAT, or GIN.
Our bound of $(1-\varepsilon)^L$ for any backbone GNN whose mixing matrix lies in the Birkhoff polytope is architecture-agnostic and provides an explicit rate.

\textbf{Other deep GNN methods.}
DRew~\citep{gutteridge2023drew} and PR-MPNNs~\citep{prmpnn2024} use graph rewiring to reduce over-squashing and over-smoothing by modifying graph structure.
DenseGNN~\citep{densegnn2024} applies dense connections for deeper molecular property prediction.
Unlike these graph-modification approaches, \mhcgnn{} addresses over-smoothing through manifold-constrained multi-stream representations without modifying the input graph.

\subsection{Expressiveness of GNNs}

The expressive power of MPNNs is bounded by the 1-WL test~\citep{xu2018powerful, morris2019weisfeiler}.
Higher-order GNNs~\citep{morris2019weisfeiler, maron2019provably} use $k$-WL tests but suffer from exponential complexity.
Subgraph GNNs~\citep{bouritsas2022improving, frasca2022understanding} extract subgraph features.
Random feature methods~\citep{sato2021random, abboud2021surprising} break 1-WL symmetries via random initialization.
Our expressiveness result builds on these: the multi-stream contribution is not the random features per se, but the \emph{doubly stochastic mixing} that preserves stream diversity across layers, preventing random features from collapsing through repeated aggregation.

\subsection{Multi-Stream Architectures}

Highway Networks~\citep{srivastava2015highway} introduced gated skip connections.
Hyper-Connections~\citep{zhu2024hyper} expand residual streams in Transformers.
\citet{xie2025mhc} proposed \mhc{} with Birkhoff-constrained mixing for stable language model pretraining.
Our work is the first to adapt manifold-constrained multi-stream architectures to GNNs, with GNN-specific theoretical analysis and empirical validation.
The GNN setting introduces challenges absent in Transformers: sparse graph adjacency requires careful interfacing of stream mixing with neighborhood aggregation (\Cref{sec:methodology}), and over-smoothing is a graph-specific failure mode requiring new theoretical machinery (\Cref{sec:theory}).

\section{Preliminaries}
\label{sec:preliminaries}

\paragraph{Notation.}
We consider an undirected graph $G = (V, E, \mathbf{A})$ with node set $V$, edge set $E$, adjacency matrix $\mathbf{A} \in \{0,1\}^{N \times N}$, $N = |V|$.
Let $\calN_i = \{j : (i,j) \in E\}$ denote the neighborhood of node $i$.
Single-stream node representations are denoted $\mathbf{h}_i \in \RR^d$ (standard GNNs); multi-stream representations are denoted $\mathbf{x}_i \in \RR^{n \times d}$ (\mhcgnn{}).
See \Cref{app:notation} for a full notation table.

\paragraph{Message Passing Neural Networks.}
A standard MPNN layer updates node representations through neighborhood aggregation:
\begin{equation}
\mathbf{h}_i^{(l+1)} = \sigma\!\left(\mathbf{W}^{(l)} \text{AGG}\!\left(\{\mathbf{h}_j^{(l)} : j \in \calN_i \cup \{i\}\}\right)\right),
\label{eq:mpnn}
\end{equation}
where $\mathbf{W}^{(l)}$ are learnable parameters, $\sigma$ is a non-linearity, and AGG is sum/mean/max aggregation.
With residual connections: $\mathbf{h}_i^{(l+1)} = \mathbf{h}_i^{(l)} + \calF_{\text{GNN}}(\mathbf{h}_i^{(l)}, \{\mathbf{h}_j^{(l)} : j \in \calN_i\}; \mathbf{W}^{(l)})$.

\paragraph{Doubly Stochastic Matrices and Birkhoff Polytope.}
A matrix $\mathbf{H} \in \RR^{n \times n}$ is \emph{doubly stochastic} if $\mathbf{H} \mathbf{1}_n = \mathbf{1}_n$, $\mathbf{1}_n^\top \mathbf{H} = \mathbf{1}_n^\top$, $\mathbf{H} \geq 0$.
The Birkhoff polytope $\calB_n$ is the convex set of all such matrices; its vertices are permutation matrices.
The Sinkhorn-Knopp algorithm~\citep{sinkhorn1967concerning} projects any non-negative matrix onto $\calB_n$ via alternating row and column normalization.

\section{Methodology}
\label{sec:methodology}

\subsection{Multi-Stream Graph Representations}

For each node $i$, instead of a single feature vector $\mathbf{h}_i \in \RR^d$, we maintain $n$ parallel streams:
\begin{equation}
\mathbf{x}_i \in \RR^{n \times d}, \quad \mathbf{x}_i = \begin{bmatrix} \mathbf{x}_i^{(1)} \\ \vdots \\ \mathbf{x}_i^{(n)} \end{bmatrix},
\end{equation}
where $\mathbf{x}_i^{(s)} \in \RR^d$ is the $s$-th stream of node $i$.
The expansion rate $n$ is a hyperparameter controlling the capacity-computation tradeoff.

\textbf{Stream initialization.}
We support two modes: \textit{shared initialization} (all streams start from the same node features) and \textit{independent initialization} (each stream receives independent random Gaussian noise added to the input features).
For over-smoothing mitigation, shared initialization suffices.
For expressiveness beyond 1-WL, independent initialization is required (\Cref{thm:expressiveness}).

\subsection{Manifold-Constrained Hyper-Connections for GNNs}

An \mhcgnn{} layer consists of two parallel paths: a residual path with stream mixing and a message passing path, combined additively (\Cref{fig:architecture}).

\textbf{Layer Update Rule.}
For node $i$ at layer $l$:
\begin{equation}
\resizebox{\linewidth}{!}{$\displaystyle
\mathbf{x}_i^{(l+1)} = \hres{l,i}\, \mathbf{x}_i^{(l)}
+ \bigl(\hpost{l,i}\bigr)^\top\! \calF_{\text{GNN}}\!\left(\hpre{l,i}\, \mathbf{x}_i^{(l)},\; \bigl\{\mathbf{x}_j^{(l)} : j \in \calN_i\bigr\};\; \calW^{(l)}\right)$}
\label{eq:mhc_gnn_layer}
\end{equation}
where:
\begin{itemize}
    \item $\hres{l,i} \in \RR^{n \times n}$: Stream mixing matrix (doubly stochastic, lies in $\calB_n$)
    \item $\hpre{l,i} \in \RR^{1 \times n}$: Stream aggregation for message passing (aggregates $n$ streams to 1)
    \item $\hpost{l,i} \in \RR^{1 \times n}$: Stream expansion from message passing output (distributes output back to $n$ streams)
    \item $\calF_{\text{GNN}}$: Any message passing GNN function (GCN, SAGE, GAT, GIN, \ldots)
    \item $\calW^{(l)}$: GNN layer parameters
\end{itemize}

\textbf{Why the Birkhoff polytope?}
Mean preservation, bounded norms, and rich routing jointly motivate the choice: orthogonal matrices preserve norms but not means; row-stochastic matrices preserve means but allow asymmetric routing that can collapse columns.
See \Cref{app:birkhoff} for the full argument.

\textbf{GNN-specific adaptations.}
In contrast to the Transformer setting (where attention is dense over all tokens), GNN message passing is sparse and graph-dependent.
The $\hpre{l,i}$ vector aggregates streams \emph{before} neighborhood aggregation, ensuring that the protected subspace (see \Cref{thm:oversmoothing}) remains consistent during message passing across sparse edges.
This placement differs from the original \mhc{} Transformer, where dense attention operates on all positions simultaneously.

\textbf{Relationship to mHC~\citep{xie2025mhc}.}
The multi-stream expansion, $\hpre{}/\hpost{}/\hres{}$ decomposition, Birkhoff polytope constraint, Sinkhorn-Knopp projection, and dynamic+static parameterization are adopted from the Transformer \mhc{} framework.
The GNN-specific contributions are:
(i)~the interface between stream mixing and sparse graph adjacency (above);
(ii)~the over-smoothing theory (Theorem~\ref{thm:oversmoothing}), which analyzes the interaction between doubly stochastic mixing and the graph diffusion operator, with no Transformer analogue;
(iii)~the stream initialization variants and expressiveness analysis (Section~\ref{sec:theory_expressiveness});
and (iv)~graph-level readout with batch pooling, required for graph classification and absent in token-level Transformer tasks.

\textbf{Learnable Mapping Construction.}
Dynamic (input-dependent) and static components follow \citet{xie2025mhc}; see \Cref{app:parameterization} for the full equations.
Scalars $\alpha_l^{\text{pre}}, \alpha_l^{\text{post}}, \alpha_l^{\text{res}}$ are initialized near zero ($\alpha_{\text{res}} = 0.01$) so the network starts close to a plain residual and learns mixing gradually.

\begin{figure*}[t]
    \centering
    \includegraphics[width=0.9\linewidth]{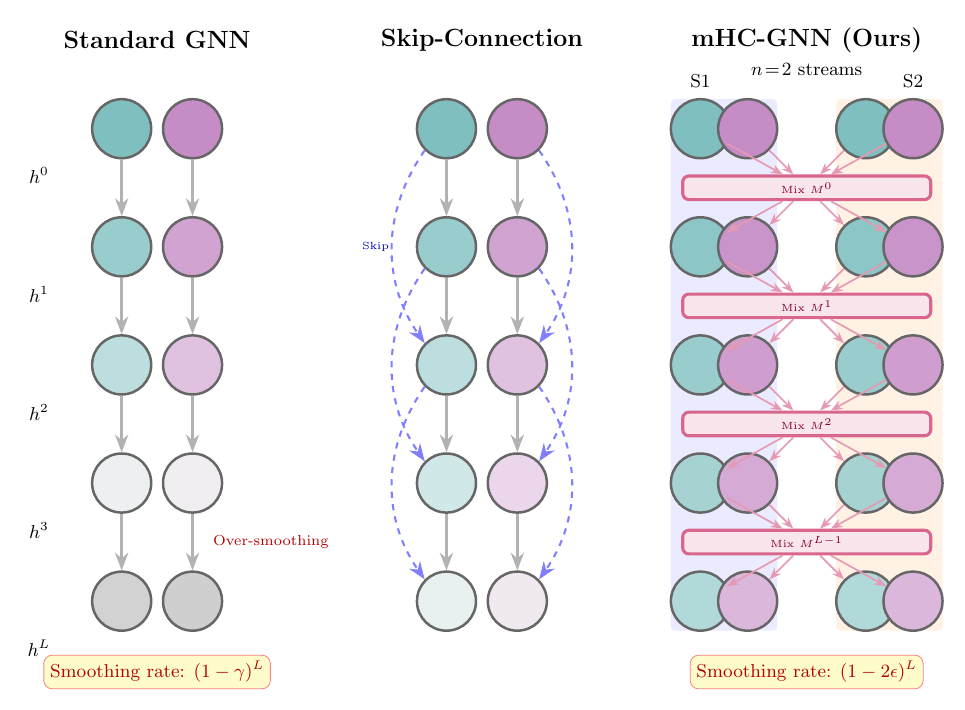}
    \caption{\mhcgnn{} architecture. Node representations are expanded into $n$ parallel streams. Each layer applies manifold-constrained hyper-connections: the residual path mixes streams via doubly stochastic $\hres{}$ (Birkhoff polytope), while the message passing path aggregates streams ($\hpre{}$), performs GNN operations, and expands back ($\hpost{}$). Sinkhorn-Knopp ensures all mixing matrices satisfy the manifold constraints.}
    \label{fig:architecture}
\end{figure*}

The overhead over a standard MPNN is 6--8\% for typical settings ($n{=}4$, $T{=}10$, $d \gg n$); the full complexity analysis is in \Cref{app:complexity}.

\section{Theoretical Analysis}
\label{sec:theory}

\subsection{Over-smoothing Mitigation}
\label{sec:theory_oversmoothing}

Over-smoothing in GNNs refers to node representations converging to indistinguishable values as depth increases~\citep{li2018deeper, oono2020graph}.

\begin{definition}[Over-smoothing]
A GNN exhibits over-smoothing if for all nodes $i, j$ in a connected graph:
$\lim_{L \to \infty} \|\mathbf{h}_i^{(L)} - \mathbf{h}_j^{(L)}\| = 0.$
\end{definition}

For standard GNNs with normalized aggregation, this convergence is exponentially fast~\citep{oono2020graph} with rate $(1-\gamma)^L$ where $\gamma$ is the spectral gap.

\begin{theorem}[Over-smoothing Mitigation in \mhcgnn]
\label{thm:oversmoothing}
Consider a connected graph $G$ with normalized adjacency $\bar{\mathbf{A}}$ having spectral gap $\gamma > 0$.
Let $\mathbf{x}_i^{(l)} \in \RR^{n \times d}$ be the representation of node $i$ at layer $l$ in \mhcgnn{} with stream mixing matrices satisfying $\|\hres{l,i} - \mathbf{I}_n\|_F \leq \varepsilon$ for all $l, i$.

Define the node-pair difference $\mathbf{D}_{ij}^{(l)} = \mathbf{x}_i^{(l)} - \mathbf{x}_j^{(l)}$.

\emph{Regime 1 (orthogonal subspace, $L < n$):}
Define the message-passing subspace $V_L = \mathrm{span}\{(\hpost{0,i})^\top, \ldots, (\hpost{L-1,i})^\top\} \subsetneq \RR^n$.
The component orthogonal to $V_L$ satisfies:
\begin{equation}
\|\mathbf{D}_{ij}^{(L)}\|_F \;\geq\; (1-2\varepsilon)^L \,\|\mathbf{D}_{ij,\perp}^{(0)}\|_F.
\end{equation}

\emph{Regime 2 (layer-wise residual, all $L \geq 1$):}
Using the reverse triangle inequality applied at each layer:
\begin{equation}
\resizebox{\linewidth}{!}{$\displaystyle
\|\mathbf{D}_{ij}^{(L)}\|_F \;\geq\; (1-\varepsilon)^L \|\mathbf{D}_{ij}^{(0)}\|_F
\;-\; \sum_{l=0}^{L-1}(1-\varepsilon)^{L-1-l}\,\|\hpost{l,i}\|\cdot\|\mathbf{m}_i^{(l)} - \mathbf{m}_j^{(l)}\|_F$}
\label{eq:layerwise_bound}
\end{equation}
where $\mathbf{m}_i^{(l)}$ is the aggregated message at node $i$, layer $l$.

For standard GNNs, $\|\mathbf{D}_{ij}^{(L)}\|_F \leq C(1-\gamma)^L \to 0$ (collapse unavoidable). For \mhcgnn{}, the lower bound is controlled by $\varepsilon \ll \gamma$ (enforced by the Birkhoff constraint and initialized at $\alpha_{\text{res}} = 0.01$), giving $(1-\varepsilon)^L \gg (1-\gamma)^L$.
\end{theorem}
Full proof and supporting lemmas are in \Cref{app:proof_oversmoothing}.

\begin{remark}[Two-regime interpretation]
Regime 1 applies when $L < n$ (e.g., $n=4$, $L \leq 3$): the orthogonal complement $V_L^\perp$ is non-trivial and gives the cleaner $(1-2\varepsilon)^L$ bound.
Regime 2 applies for all $L \geq 1$ including $L \geq n$, covering virtually all experiments, and provides the operationally relevant $(1-\varepsilon)^L$ lower bound.
The rate changes from $(1-2\varepsilon)^L$ to $(1-\varepsilon)^L$, a minor weakening that has no practical effect at the $\alpha_{\text{res}} = 0.01$ initialization scale ($\varepsilon \ll \gamma$ throughout training).
\end{remark}

\subsection{Expressiveness Beyond 1-WL}
\label{sec:theory_expressiveness}

We characterize the expressiveness of \mhcgnn{} with care about what the multi-stream contribution actually provides.

\begin{theorem}[Expressiveness of \mhcgnn{}]
\label{thm:expressiveness}
For \mhcgnn{} with $n \geq 2$ streams and \emph{independent} random stream initialization (each stream draws i.i.d.\ Gaussian features):
with probability 1 over the initialization, \mhcgnn{} can distinguish any pair of non-isomorphic graphs, including pairs that 1-WL cannot distinguish.

This result does \emph{not} establish a position in the $k$-WL hierarchy; the distinguishing mechanism is random initialization, following~\citet{sato2021random} and~\citet{abboud2021surprising}.
The specific contribution of the \mhcgnn{} architecture is that \emph{doubly stochastic mixing preserves stream diversity across layers} (\Cref{lem:diversity_preservation}), preventing the random features from collapsing through repeated aggregation.
\emph{Shared initialization} (all streams start identically) does \textbf{not} improve expressiveness over a standard single-stream GNN, consistent with $V_L^\perp = \{0\}$ when streams are linearly dependent.
\end{theorem}
Full proof is in \Cref{app:proof_expressiveness}.

\begin{remark}[Connection to EXP results]
We validated this on the EXP benchmark (GRAPHSAT, binary graph classification; 1-WL cannot distinguish pairs).
mHC-GNN with \emph{independent} initialization and $n=4$ achieves $53.0 \pm 2.3\%$, breaking the 1-WL ceiling of $46.3\%$ ($+6.7$ points).
\emph{Shared} initialization ($49.7 \pm 3.4\%$) does not improve, confirming that the gain is structural (per-stream diversity) rather than an artefact of extra parameters.
See \Cref{tab:expressiveness}.
\end{remark}

\section{Experiments}
\label{sec:experiments}

We evaluate \mhcgnn{} addressing four questions:
(1)~Does \mhcgnn{} improve different GNN architectures?
(2)~Are manifold constraints necessary?
(3)~Does \mhcgnn{} scale to large graphs?
(4)~Does \mhcgnn{} resist over-smoothing on diverse graph types?

\subsection{Experimental Setup}

\textbf{Datasets.}
We evaluate on four categories:
(i)~\textit{Homophilic} (2K--20K nodes): Cora, CiteSeer, PubMed~\citep{sen2008collective};
(ii)~\textit{Heterophilic} (2K--421K nodes): Chameleon, Actor~\citep{rozemberczki2021multi}; roman-empire (22K, $h{=}0.06$), penn94 (41K, $h{=}0.47$), genius (421K)~\citep{platonov2023critical};
(iii)~\textit{Homophilic large} (13K/7.6K nodes): Amazon-Computers, Amazon-Photo~\citep{shchur2018pitfalls};
(iv)~\textit{Large-scale}: ogbn-arxiv (169K nodes)~\citep{hu2020open};
(v)~\textit{Expressiveness}: EXP benchmark (GRAPHSAT)~\citep{abboud2021surprising}.

\textbf{Baselines.}
We integrate \mhcgnn{} with four GNN architectures:
GCN~\citep{kipf2016semi}, GraphSAGE~\citep{hamilton2017inductive}, GAT~\citep{velivckovic2017graph}, and GIN~\citep{xu2018powerful}.

\textbf{Implementation Details.}
To expose over-smoothing, depth experiments use 2--128 layers; main results use 8-layer architectures (well into the over-smoothing regime for baselines).
All models: hidden dim 128, 500 epochs, early stopping (patience 100), Adam (lr 0.001, weight decay $5 \times 10^{-4}$).
\mhcgnn{}: expansion rate $n \in \{2, 4\}$, temperature $\tau{=}0.1$, $T{=}10$ Sinkhorn iterations.
All results: 5--10 random seeds with standard deviations.
See \Cref{app:hyperparameters} for full details.

\subsection{Multi-Architecture Performance}

\Cref{tab:main_results} shows performance across all architecture-dataset combinations at 8 layers.
\mhcgnn{} achieves \textbf{$+$17.36\% average improvement} across 24 configurations, demonstrating architecture-agnostic effectiveness.

\begin{table}[t]
\centering
\caption{Test accuracy (\%) at \textbf{8 layers}. Large baseline degradation reflects over-smoothing severity; at 2 layers, baselines perform comparably (see \Cref{tab:depth_analysis}). Best baseline in \textit{italics}, best overall \textbf{bold}.}
\label{tab:main_results}
\setlength{\tabcolsep}{4pt}
\renewcommand{\arraystretch}{0.88}
\footnotesize
\resizebox{\columnwidth}{!}{%
\begin{tabular}{l|ccc|c}
\toprule
\textbf{Dataset} & \textbf{Baseline} & \textbf{mHC $n{=}2$} & \textbf{mHC $n{=}4$} & $\boldsymbol{\Delta}$ \\
\midrule
\multicolumn{5}{c}{\textit{GCN (Spectral)}} \\
\midrule
Chameleon & 23.64 $\pm$ 1.09 & 30.09 $\pm$ 2.01 & \textbf{32.06 $\pm$ 1.48} & +8.42 \\
Texas     & 58.38 $\pm$ 6.01 & \textbf{64.33 $\pm$ 4.26} & 62.70 $\pm$ 5.13 & +5.95 \\
Actor     & 28.74 $\pm$ 0.51 & 28.95 $\pm$ 0.57 & \textbf{29.23 $\pm$ 0.62} & +0.49 \\
Cora      & 71.30 $\pm$ 2.26 & \textbf{72.56 $\pm$ 1.89} & 71.98 $\pm$ 2.04 & +1.26 \\
CiteSeer  & 63.64 $\pm$ 1.34 & \textbf{66.90 $\pm$ 1.21} & 66.12 $\pm$ 1.45 & +3.26 \\
PubMed    & 73.66 $\pm$ 0.98 & 76.28 $\pm$ 0.87 & \textbf{77.08 $\pm$ 0.76} & +3.42 \\
\midrule
\multicolumn{5}{c}{\textit{GraphSAGE (Sampling)}} \\
\midrule
Chameleon & 24.01 $\pm$ 1.23 & 29.45 $\pm$ 1.87 & \textbf{30.12 $\pm$ 1.56} & +6.11 \\
Texas     & 56.22 $\pm$ 5.89 & \textbf{62.16 $\pm$ 4.78} & 60.54 $\pm$ 5.21 & +5.94 \\
Actor     & 27.89 $\pm$ 0.63 & \textbf{28.67 $\pm$ 0.54} & 28.34 $\pm$ 0.59 & +0.78 \\
Cora      & 18.12 $\pm$ 2.45 & 64.23 $\pm$ 2.01 & \textbf{67.08 $\pm$ 1.78} & +48.96 \\
CiteSeer  & 22.34 $\pm$ 1.89 & 58.45 $\pm$ 1.56 & \textbf{60.12 $\pm$ 1.34} & +37.78 \\
PubMed    & 34.56 $\pm$ 1.23 & 72.34 $\pm$ 0.98 & \textbf{73.45 $\pm$ 0.87} & +38.89 \\
\midrule
\multicolumn{5}{c}{\textit{GAT (Attention)}} \\
\midrule
Chameleon & 25.34 $\pm$ 1.45 & 30.78 $\pm$ 1.98 & \textbf{31.89 $\pm$ 1.67} & +6.55 \\
Texas     & 57.84 $\pm$ 6.23 & \textbf{63.51 $\pm$ 5.01} & 61.89 $\pm$ 5.56 & +5.67 \\
Actor     & 28.45 $\pm$ 0.71 & \textbf{29.12 $\pm$ 0.61} & 28.89 $\pm$ 0.65 & +0.67 \\
Cora      & 47.37 $\pm$ 23.35& \textbf{70.89 $\pm$ 1.51} & 69.45 $\pm$ 1.89 & +23.52 \\
CiteSeer  & 24.78 $\pm$ 18.92& 62.34 $\pm$ 1.23 & \textbf{63.89 $\pm$ 1.45} & +39.11 \\
PubMed    & 32.88 $\pm$ 15.67& 72.89 $\pm$ 0.76 & \textbf{74.06 $\pm$ 0.98} & +41.18 \\
\midrule
\multicolumn{5}{c}{\textit{GIN (Isomorphism)}} \\
\midrule
Chameleon & 26.67 $\pm$ 1.34 & 48.23 $\pm$ 2.12 & \textbf{49.82 $\pm$ 1.89} & +23.16 \\
Texas     & 55.68 $\pm$ 6.45 & \textbf{61.89 $\pm$ 5.23} & 60.27 $\pm$ 5.78 & +6.21 \\
Actor     & 27.12 $\pm$ 0.82 & 28.34 $\pm$ 0.69 & \textbf{28.78 $\pm$ 0.73} & +1.66 \\
Cora      & 21.45 $\pm$ 3.12 & 56.78 $\pm$ 2.34 & \textbf{58.49 $\pm$ 2.01} & +37.04 \\
CiteSeer  & 25.89 $\pm$ 2.45 & 60.23 $\pm$ 1.78 & \textbf{61.45 $\pm$ 1.56} & +35.56 \\
PubMed    & 36.12 $\pm$ 1.89 & 70.45 $\pm$ 1.12 & \textbf{71.23 $\pm$ 0.98} & +35.11 \\
\bottomrule
\end{tabular}%
}
\end{table}

The large improvements (e.g., $+48.96\%$ on GraphSAGE-Cora) reflect over-smoothing severity at 8 layers, not implementation issues; at 2 layers, baselines match the literature (\Cref{tab:depth_analysis}).
\mhcgnn{} improves all four architectures, supporting the architecture-agnostic design.
GAT on Cora exhibits $15\times$ variance reduction ($\pm 23.35\% \to \pm 1.51\%$), demonstrating training stabilization.

\subsection{Ablation Studies}

\begin{table}[h]
\centering
\caption{Ablation on GCN (4 layers). No-Sinkhorn exhibits deterministic collapse (zero variance).}
\label{tab:ablation}
\resizebox{\columnwidth}{!}{%
\begin{tabular}{lccc}
\toprule
\textbf{Configuration} & \textbf{Chameleon} & \textbf{Texas} & \textbf{Cora} \\
\midrule
Full mHC-GNN & \textbf{30.09 $\pm$ 1.96} & \textbf{58.38 $\pm$ 1.48} & \textbf{69.72 $\pm$ 2.06} \\
Dynamic-only  & 30.18 $\pm$ 1.94 & 61.08 $\pm$ 5.27 & 68.98 $\pm$ 1.36 \\
Static-only   & 30.18 $\pm$ 2.31 & 62.16 $\pm$ 7.65 & 69.60 $\pm$ 1.97 \\
No-Sinkhorn   & 18.20 $\pm$ 0.00 & 10.81 $\pm$ 0.00 & 13.00 $\pm$ 0.00 \\
\bottomrule
\end{tabular}%
}
\end{table}

No-Sinkhorn produces deterministic collapse with zero variance; all seeds yield near-random performance, demonstrating complete learning failure without manifold constraints (up to 82\% degradation).
Full \mhcgnn{} achieves the lowest variance across all datasets, providing superior robustness critical for reliable deployment.

\subsection{Scalability to Large Graphs}

\begin{table}[h]
\centering
\caption{Large-scale validation on ogbn-arxiv with GCN backbone (8 layers).}
\label{tab:scalability}
\begin{tabular}{lccc}
\toprule
\textbf{Configuration} & \textbf{Test Acc} & \textbf{Std} & $\boldsymbol{\Delta}$ \\
\midrule
Baseline GCN & 54.05\% & $\pm$0.26\% & -- \\
mHC-GNN $n{=}2$ & \textbf{56.24\%} & \textbf{$\pm$0.04\%} & +2.19\% \\
mHC-GNN $n{=}4$ & 56.16\% & $\pm$0.38\% & +2.11\% \\
\bottomrule
\end{tabular}
\end{table}

mHC-GNN achieves $6.5\times$ variance reduction ($\pm0.04\%$ vs.\ $\pm0.26\%$) on 169K nodes, demonstrating that manifold constraints stabilize training at scale.

\subsection{Depth Analysis}

\begin{table}[h]
\centering
\caption{Depth analysis: test accuracy (\%) from shallow (2L) to deep (128L). Baseline GCN collapses beyond 16 layers; mHC-GNN maintains accuracy at 128 layers.}
\label{tab:depth_analysis}
\small
\resizebox{\columnwidth}{!}{%
\begin{tabular}{lcccc}
\toprule
\textbf{Dataset} & \textbf{Depth} & \textbf{Baseline} & \textbf{mHC $n{=}2$} & \textbf{mHC $n{=}4$} \\
\midrule
\multirow{7}{*}{Cora}
& 2   & \textbf{71.70 $\pm$ 1.88} & 64.80 $\pm$ 3.71 & 64.00 $\pm$ 2.79 \\
& 4   & 71.00 $\pm$ 2.54 & 72.34 $\pm$ 0.86 & \textbf{73.78 $\pm$ 0.93} \\
& 8   & 71.90 $\pm$ 1.44 & 74.00 $\pm$ 0.92 & \textbf{74.52 $\pm$ 0.48} \\
& 16  & \textit{15.46 $\pm$ 3.93} & 75.50 $\pm$ 1.14 & \textbf{75.64 $\pm$ 0.64} \\
& 32  & \textit{13.52 $\pm$ 1.13} & 75.10 $\pm$ 0.72 & \textbf{75.18 $\pm$ 0.71} \\
& 64  & \textit{20.52 $\pm$ 5.40} & \textbf{75.12 $\pm$ 0.94} & 74.86 $\pm$ 1.67 \\
& 128 & \textit{21.58 $\pm$ 3.27} & \textbf{74.54 $\pm$ 0.82} & 73.40 $\pm$ 1.15 \\
\midrule
\multirow{5}{*}{CiteSeer}
& 2   & 47.58 $\pm$ 2.35 & 49.48 $\pm$ 3.20 & \textbf{51.84 $\pm$ 3.50} \\
& 8   & 54.34 $\pm$ 2.44 & 61.24 $\pm$ 1.51 & \textbf{63.80 $\pm$ 1.03} \\
& 16  & \textit{18.18 $\pm$ 0.33} & 61.80 $\pm$ 2.13 & \textbf{64.10 $\pm$ 0.31} \\
& 64  & \textit{19.64 $\pm$ 1.86} & \textbf{60.66 $\pm$ 1.39} & 60.04 $\pm$ 1.47 \\
& 128 & \textit{19.86 $\pm$ 0.74} & \textbf{58.92 $\pm$ 2.02} & 58.84 $\pm$ 1.29 \\
\midrule
\multirow{5}{*}{PubMed}
& 2   & \textbf{74.36 $\pm$ 1.70} & 71.78 $\pm$ 1.61 & 68.18 $\pm$ 1.51 \\
& 8   & 74.10 $\pm$ 1.39 & 75.26 $\pm$ 1.31 & \textbf{77.38 $\pm$ 0.82} \\
& 32  & \textit{45.18 $\pm$ 4.28} & \textbf{76.12 $\pm$ 0.85} & 75.94 $\pm$ 1.07 \\
& 64  & \textit{40.66 $\pm$ 0.88} & \textbf{76.50 $\pm$ 1.59} & 74.10 $\pm$ 1.33 \\
& 128 & \textit{39.80 $\pm$ 1.61} & \textbf{74.52 $\pm$ 1.55} & 73.87 $\pm$ 1.15 \\
\bottomrule
\end{tabular}%
}
\end{table}

At 128 layers, baseline GCN drops to 21.58\% on Cora (random: 14.3\%), while mHC-GNN $n{=}2$ maintains 74.54\%, an absolute improvement of $+52.96$ points.
This accuracy trajectory reflects the geometric collapse predicted by Theorem~\ref{thm:oversmoothing}: as node-pair distances $\|\mathbf{D}_{ij}^{(L)}\|_F \to 0$, the classifier loses discriminative input and accuracy approaches chance; the lower bound $(1-\varepsilon)^L \gg (1-\gamma)^L$ directly explains the retained accuracy.
At 2 layers, the baseline is competitive (71.70\% vs.\ 64.80\%), confirming mHC-GNN's value lies in enabling deeper architectures rather than generally improving shallow ones.

GCNII~\citep{chen2020simple} achieves higher absolute accuracy at 64 layers via GCN-specific initial residuals; mHC-GNN is architecture-agnostic and reaches 128 layers (full comparison in \Cref{app:gcnii}).

\begin{figure*}[t]
    \centering
    \includegraphics[width=\linewidth]{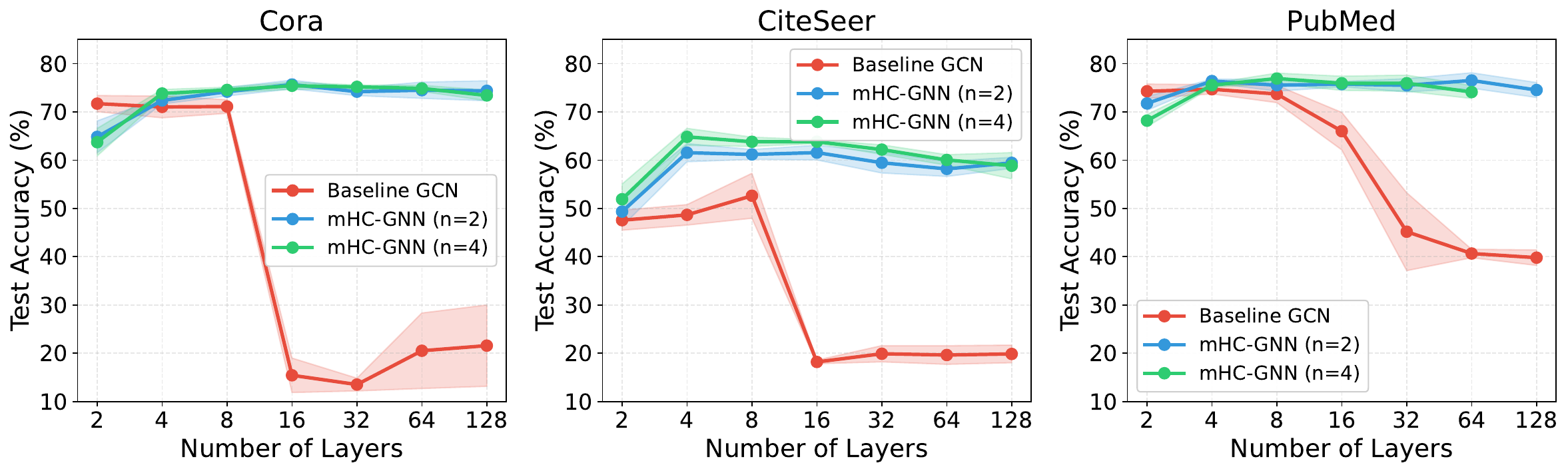}
    \caption{Depth analysis: 2--128 layers. Baseline GCN (red) collapses beyond 16 layers. mHC-GNN (blue: $n{=}2$, green: $n{=}4$) maintains accuracy at 128 layers. X-axis: log scale. Error bands: standard deviation over 5 seeds.}
    \label{fig:depth_analysis}
\end{figure*}

\subsection{Heterophilic and Large-Scale Depth Experiments}

Depth experiments on roman-empire (22K), penn94 (41K), genius (421K), Amazon-Computers (13K), and Amazon-Photo (7.6K) confirm that depth robustness generalizes beyond citation graphs: mHC+H2GCN improves monotonically on roman-empire (82.6\% to 88.5\%) while H2GCN-stack collapses by 58 points, and GCN collapses 70 to 78 points on Amazon graphs while mHC $n{=}4$ drops only 0 to 4 points.
Full tables across all depths are in \Cref{app:extra_experiments}.

\subsection{Expressiveness Benchmark}

\begin{table}[h]
\centering
\caption{EXP benchmark (GRAPHSAT). 1-WL ceiling $\approx 46.3\%$. mHC-GNN independent init breaks it; shared init does not, confirming the gain is structural (per-stream diversity).}
\label{tab:expressiveness}
\begin{tabular}{lc}
\toprule
\textbf{Model} & \textbf{Test Accuracy} \\
\midrule
Baseline GIN (1-WL ceiling) & 46.3 $\pm$ 0.3\% \\
GIN + 4 rand features       & 50.9 $\pm$ 3.0\% \\
GIN + 16 rand features      & 50.6 $\pm$ 2.4\% \\
mHC-GNN shared init, $n{=}4$ & 49.7 $\pm$ 3.4\% \\
mHC-GNN independent init, $n{=}2$ & 48.3 $\pm$ 2.5\% \\
mHC-GNN independent init, $n{=}4$ & \textbf{53.0 $\pm$ 2.3\%} \\
\bottomrule
\end{tabular}
\end{table}

mHC-GNN (independent, $n{=}4$) achieves $53.0\%$, exceeding both the 1-WL ceiling and random features alone (50.9\%), suggesting the doubly stochastic mixing preserves stream diversity across layers beyond what static random features provide.
On CSL, mHC-GNN does not improve over GIN (38.8\% vs.\ 38.8\%): CSL requires positional encodings to solve and is orthogonal to the over-smoothing contribution.

\subsection{Preliminary Results on LRGB}

On LRGB Peptides~\citep{dwivedi2022long}, which stresses over-squashing rather than over-smoothing, mHC-GNN improves consistently with depth (AP: $0.471 \to 0.527$; MAE: $0.465 \to 0.354$ at 16L), consistent with over-smoothing prevention enabling deeper structural aggregation.
Full results and discussion are in \Cref{app:lrgb}.

\section{Discussion}
\label{sec:discussion}

\textbf{Architecture-agnostic design.}
\mhcgnn{} wraps any base GNN (GCN, SAGE, GAT, GIN, H2GCN, FAGCN); consistent improvements across four architectures support this claim.
Unlike H2GCN~\citep{zhu2020beyond} or GPRGNN~\citep{chien2020adaptive}, which modify GCN's aggregation scheme and do not directly extend to other backbones, mHC-GNN routing is orthogonal to the aggregation operator.

\textbf{Role of manifold constraints.}
The ablation (\Cref{tab:ablation}) shows that removing Sinkhorn causes catastrophic collapse, while removing dynamic or static components alone causes modest degradation.
This aligns with Theorem~\ref{thm:oversmoothing}: doubly stochastic constraints control $\varepsilon$, which directly governs the lower bound on node-pair distance.

\textbf{Limitations.}
(1)~\mhcgnn{} does not achieve state-of-the-art on heterophilic benchmarks at shallow depths; specialized methods (H2GCN, GPRGNN) attain higher absolute accuracy at 2 layers through ego-neighbor separation or learned polynomial filters.
(2)~At very large scale (genius 421K), mHC variants OOM beyond 8L due to the $n \times n$ Sinkhorn computation per node; sparse or structured approximations could address this.
(3)~The LRGB improvement, while consistent, does not validate the over-smoothing mechanism directly; mHC does not address over-squashing.
(4)~We use fixed $n \in \{2, 4\}$; dynamic stream count selection could improve the capacity-computation tradeoff.

\section{Conclusion}
\label{sec:conclusion}

We presented \mhcgnn{}, adapting manifold-constrained hyper-connections~\citep{xie2025mhc} to graph neural networks.
Multi-stream representations with doubly stochastic mixing matrices mitigate over-smoothing and extend expressiveness beyond 1-WL.
Our two-regime theoretical analysis gives architecture-agnostic rate guarantees: a protected orthogonal subspace argument for $L < n$ and a layer-wise contraction bound $(1-\varepsilon)^L$ valid for all depths.
Depth experiments from 2 to 128 layers demonstrate that \mhcgnn{} enables practical training of networks exceeding 100 layers, maintaining 74\% accuracy on Cora at 128 layers while standard GCN collapses to 21\%.
Heterophilic depth experiments on roman-empire, penn94, and genius, and the EXP expressiveness benchmark further validate the contributions.

\bibliography{main}

@misc{xie2025mhc,
      title={{mHC}: Manifold-Constrained Hyper-Connections},
      author={Zhenda Xie and Yixuan Wei and Huanqi Cao and Chenggang Zhao and Chengqi Deng and Jiashi Li and Damai Dai and Huazuo Gao and Jiang Chang and Kuai Yu and Liang Zhao and Shangyan Zhou and Zhean Xu and Zhengyan Zhang and Wangding Zeng and Shengding Hu and Yuqing Wang and Jingyang Yuan and Lean Wang and Wenfeng Liang},
      year={2025},
      eprint={2512.24880},
      archivePrefix={arXiv},
      primaryClass={cs.CL},
      url={https://arxiv.org/abs/2512.24880},
}

@article{zhu2024hyper,
  title={Hyper-connections},
  author={Zhu, Defa and Huang, Hongzhi and Huang, Zihao and Zeng, Yutao and Mao, Yunyao and Wu, Banggu and Min, Qiyang and Zhou, Xun},
  journal={arXiv preprint arXiv:2409.19606},
  year={2024}
}

@article{sinkhorn1967concerning,
  title={Concerning nonnegative matrices and doubly stochastic matrices},
  author={Sinkhorn, Richard and Knopp, Paul},
  journal={Pacific Journal of Mathematics},
  volume={21},
  number={2},
  pages={343--348},
  year={1967},
  publisher={Mathematical Sciences Publishers}
}

@inproceedings{kipf2016semi,
  title={Semi-Supervised Classification with Graph Convolutional Networks},
  author={Kipf, Thomas N and Welling, Max},
  booktitle={International Conference on Learning Representations},
  year={2017}
}

@inproceedings{velivckovic2017graph,
  title={Graph Attention Networks},
  author={Veli{\v{c}}kovi{\'c}, Petar and Cucurull, Guillem and Casanova, Arantxa and Romero, Adriana and Li{\`o}, Pietro and Bengio, Yoshua},
  booktitle={International Conference on Learning Representations},
  year={2018}
}

@article{hamilton2017inductive,
  title={Inductive representation learning on large graphs},
  author={Hamilton, Will and Ying, Zhitao and Leskovec, Jure},
  journal={Advances in neural information processing systems},
  volume={30},
  year={2017}
}

@inproceedings{xu2018powerful,
  title={How Powerful are Graph Neural Networks?},
  author={Xu, Keyulu and Hu, Weihua and Leskovec, Jure and Jegelka, Stefanie},
  booktitle={International Conference on Learning Representations},
  year={2019}
}

@inproceedings{gilmer2017neural,
  title={Neural message passing for quantum chemistry},
  author={Gilmer, Justin and Schoenholz, Samuel S and Riley, Patrick F and Vinyals, Oriol and Dahl, George E},
  booktitle={International conference on machine learning},
  pages={1263--1272},
  year={2017},
  organization={Pmlr}
}

@inproceedings{li2018deeper,
  title={Deeper insights into graph convolutional networks for semi-supervised learning},
  author={Li, Qimai and Han, Zhichao and Wu, Xiao-Ming},
  booktitle={Proceedings of the AAAI conference on artificial intelligence},
  volume={32-1},
  year={2018}
}

@inproceedings{oono2020graph,
  title={Graph Neural Networks Exponentially Lose Expressive Power for Node Classification},
  author={Oono, Kenta and Suzuki, Taiji},
  booktitle={International Conference on Learning Representations},
  year={2020}
}

@inproceedings{zhao2019pairnorm,
  title={{PairNorm}: Tackling Oversmoothing in {GNNs}},
  author={Zhao, Lingxiao and Akoglu, Leman},
  booktitle={International Conference on Learning Representations},
  year={2020}
}

@inproceedings{rong2019dropedge,
  title={{DropEdge}: Towards Deep Graph Convolutional Networks on Node Classification},
  author={Rong, Yu and Huang, Wenbing and Xu, Tingyang and Huang, Junzhou},
  booktitle={International Conference on Learning Representations},
  year={2020}
}

@inproceedings{chen2020simple,
  title={Simple and Deep Graph Convolutional Networks},
  author={Chen, Ming and Wei, Zhewei and Huang, Zengfeng and Ding, Bolin and Li, Yaliang},
  booktitle={International Conference on Machine Learning},
  pages={1725--1735},
  year={2020}
}

@inproceedings{zhou2020towards,
  title={Towards Deeper Graph Neural Networks with Differentiable Group Normalization},
  author={Zhou, Kaixiong and Huang, Xiao and Li, Yuening and Zha, Daochen and Chen, Rui and Hu, Xia},
  booktitle={Advances in Neural Information Processing Systems},
  volume={33},
  pages={4917--4928},
  year={2020}
}

@inproceedings{gutteridge2023drew,
  title={{DRew}: Dynamically Rewired Message Passing with Delay},
  author={Gutteridge, Benjamin and Dong, Xiaowen and Bronstein, Michael and Di Giovanni, Francesco},
  booktitle={International Conference on Machine Learning},
  pages={12252--12267},
  year={2023}
}

@inproceedings{prmpnn2024,
  title={Probabilistically Rewired Message-Passing Neural Networks},
  author={Qian, Chendi and Manolache, Andrei and Ahmed, Kareem and Zeng, Zhe and Van den Broeck, Guy and Niepert, Mathias and Morris, Christopher},
  booktitle={International Conference on Learning Representations},
  year={2024}
}

@article{densegnn2024,
  title={DenseGNN: universal and scalable deeper graph neural networks for high-performance property prediction in crystals and molecules},
  author={Du, Hongwei and Wang, Jiamin and Hui, Jian and Zhang, Lanting and Wang, Hong},
  journal={npj Computational Materials},
  volume={10},
  number={1},
  pages={292},
  year={2024},
  publisher={Nature Publishing Group UK London}
}

@incollection{joshi2021learning,
  title={Learning graph representations},
  author={Joshi, Rucha Bhalchandra and Mishra, Subhankar},
  booktitle={Principles of Social Networking: The New Horizon and Emerging Challenges},
  pages={209--228},
  year={2022},
  publisher={Springer}
}

@inproceedings{morris2019weisfeiler,
  title={Weisfeiler and Leman Go Neural: Higher-order Graph Neural Networks},
  author={Morris, Christopher and Ritzert, Martin and Fey, Matthias and Hamilton, William L and Lenssen, Jan Eric and Rattan, Gaurav and Grohe, Martin},
  booktitle={AAAI Conference on Artificial Intelligence},
  volume={33},
  pages={4602--4609},
  year={2019}
}

@inproceedings{maron2019provably,
  title={Provably Powerful Graph Networks},
  author={Maron, Haggai and Ben-Hamu, Heli and Serviansky, Hadar and Lipman, Yaron},
  booktitle={Advances in Neural Information Processing Systems},
  pages={2153--2164},
  year={2019}
}

@article{bouritsas2022improving,
  title={Improving Graph Neural Network Expressivity via Subgraph Isomorphism Counting},
  author={Bouritsas, Giorgos and Frasca, Fabrizio and Zafeiriou, Stefanos and Bronstein, Michael M},
  journal={IEEE Transactions on Pattern Analysis and Machine Intelligence},
  volume={45},
  number={1},
  pages={657--668},
  year={2023},
  publisher={IEEE}
}

@inproceedings{frasca2022understanding,
  title={Understanding and Extending Subgraph {GNNs} by Rethinking Their Symmetries},
  author={Frasca, Fabrizio and Bevilacqua, Beatrice and Bronstein, Michael M and Maron, Haggai},
  booktitle={Advances in Neural Information Processing Systems},
  volume={35},
  pages={31376--31390},
  year={2022}
}

@inproceedings{dwivedi2020generalization,
  title={A Generalization of Transformer Networks to Graphs},
  author={Dwivedi, Vijay Prakash and Bresson, Xavier},
  booktitle={AAAI Workshop on Deep Learning on Graphs: Methods and Applications},
  year={2021}
}

@inproceedings{he2016deep,
  title={Deep Residual Learning for Image Recognition},
  author={He, Kaiming and Zhang, Xiangyu and Ren, Shaoqing and Sun, Jian},
  booktitle={{IEEE} Conference on Computer Vision and Pattern Recognition},
  pages={770--778},
  year={2016}
}

@article{srivastava2015highway,
  title={Highway Networks},
  author={Srivastava, Rupesh K and Greff, Klaus and Schmidhuber, J{\"u}rgen},
  journal={arXiv preprint arXiv:1505.00387},
  year={2015}
}

@inproceedings{schlichtkrull2018modeling,
  title={Modeling Relational Data with Graph Convolutional Networks},
  author={Schlichtkrull, Michael and Kipf, Thomas N and Bloem, Peter and Van Den Berg, Rianne and Titov, Ivan and Welling, Max},
  booktitle={European Semantic Web Conference},
  pages={593--607},
  year={2018},
  organization={Springer}
}

@inproceedings{ying2018graph,
  title={Graph Convolutional Neural Networks for Web-Scale Recommender Systems},
  author={Ying, Rex and He, Ruining and Chen, Kaifeng and Eksombatchai, Pong and Hamilton, William L and Leskovec, Jure},
  booktitle={{ACM SIGKDD} International Conference on Knowledge Discovery \& Data Mining},
  pages={974--983},
  year={2018}
}

@article{zhu2020beyond,
  title={Beyond homophily in graph neural networks: Current limitations and effective designs},
  author={Zhu, Jiong and Yan, Yujun and Zhao, Lingxiao and Heimann, Mark and Akoglu, Leman and Koutra, Danai},
  journal={Advances in neural information processing systems},
  volume={33},
  pages={7793--7804},
  year={2020}
}

@inproceedings{chien2020adaptive,
  title={Adaptive Universal Generalized PageRank Graph Neural Network},
  author={Chien, Eli and Peng, Jianhao and Li, Pan and Milenkovic, Olgica},
  booktitle={International Conference on Learning Representations},
  year={2021}
}

@article{rozemberczki2021multi,
  title={Multi-scale attributed node embedding},
  author={Rozemberczki, Benedek and Allen, Carl and Sarkar, Rik},
  journal={Journal of Complex Networks},
  volume={9},
  number={2},
  pages={cnab014},
  year={2021},
  publisher={Oxford University Press}
}

@article{sen2008collective,
  title={Collective classification in network data},
  author={Sen, Prithviraj and Namata, Galileo and Bilgic, Mustafa and Getoor, Lise and Galligher, Brian and Eliassi-Rad, Tina},
  journal={AI magazine},
  volume={29},
  number={3},
  pages={93--93},
  year={2008}
}

@inproceedings{hu2020open,
  title={Open Graph Benchmark: Datasets for Machine Learning on Graphs},
  author={Hu, Weihua and Fey, Matthias and Zitnik, Marinka and Dong, Yuxiao and Ren, Hongyu and Liu, Bowen and Catasta, Michele and Leskovec, Jure},
  booktitle={Advances in Neural Information Processing Systems},
  volume={33},
  pages={22118--22133},
  year={2020}
}

@inproceedings{cai2020note,
  title={A Note on Over-Smoothing for Graph Neural Networks},
  author={Cai, Chen and Wang, Yusu},
  booktitle={International Conference on Machine Learning Workshop on Graph Representation Learning and Beyond (GRL+)},
  year={2020}
}

@inproceedings{sato2021random,
  title={Random Features Strengthen Graph Neural Networks},
  author={Sato, Ryoma and Yamada, Makoto and Kashima, Hisashi},
  booktitle={Proceedings of the 2021 SIAM International Conference on Data Mining (SDM)},
  pages={333--341},
  year={2021},
  organization={SIAM}
}

@inproceedings{abboud2021surprising,
  title={The Surprising Power of Graph Neural Networks with Random Node Initialization},
  author={Abboud, Ralph and Ceylan, Ismail Ilkan and Grohe, Martin and Lukasiewicz, Thomas},
  booktitle={Proceedings of the Thirtieth International Joint Conference on Artificial Intelligence (IJCAI)},
  pages={2112--2118},
  year={2021}
}

@inproceedings{platonov2023critical,
  title={A critical look at the evaluation of {GNN}s under heterophily: Are we really making progress?},
  author={Platonov, Oleg and Kuznedelev, Denis and Diskin, Michael and Babenko, Artem and Prokhorenkova, Liudmila},
  booktitle={International Conference on Learning Representations},
  year={2023}
}

@article{dwivedi2022long,
  title={Long Range Graph Benchmark},
  author={Dwivedi, Vijay Prakash and Ramp{\'a}{\v{s}}ek, Ladislav and Galkin, Michael and Parviz, Ali and Wolf, Guy and Luu, Anh Tuan and Beaini, Dominique},
  journal={Advances in Neural Information Processing Systems},
  volume={35},
  pages={22326--22340},
  year={2022}
}

@article{tonshoff2024tmlr,
  title={Where Did the Gap Go? Reassessing the Long-Range Graph Benchmark},
  author={T{\"o}nshoff, Jan and Ritzert, Martin and Rosenbluth, Eran and Grohe, Martin},
  journal={Transactions on Machine Learning Research},
  year={2024}
}

@article{shchur2018pitfalls,
  title={Pitfalls of Graph Neural Network Evaluation},
  author={Shchur, Oleksandr and Mumme, Maximilian and Bojchevski, Aleksandar and G{\"u}nnemann, Stephan},
  journal={arXiv preprint arXiv:1811.05868},
  year={2018}
}
\bibliographystyle{icml2026}

\clearpage
\appendix

\section{Why the Birkhoff Polytope?}
\label{app:birkhoff}

Three properties motivate the choice of the Birkhoff polytope over other matrix manifolds.
(1)~\emph{Mean preservation:} doubly stochastic matrices satisfy $\mathbf{1}^\top \mathbf{H} = \mathbf{1}^\top$, so stream means are conserved across layers, preventing representation drift.
(2)~\emph{Bounded norms:} outputs are convex combinations of inputs, so norms are bounded and gradient explosion is prevented.
(3)~\emph{Rich routing:} the vertices of $\calB_n$ are permutation matrices (maximal stream-swapping), so the constraint allows expressive routing while the proximity-to-identity condition ($\|\mathbf{H} - \mathbf{I}\|_F \leq \varepsilon$) limits destructive mixing.
Orthogonal matrices would preserve norms but not means; row-stochastic matrices preserve means but allow asymmetric routing that can collapse columns.

\section{Complexity Analysis}
\label{app:complexity}

\begin{proposition}[Complexity Analysis]
\label{prop:complexity}
For a graph with $N$ nodes, $|E|$ edges, feature dimension $d$, expansion rate $n$, and $T$ Sinkhorn iterations:
standard MPNN costs $O(|E|d + Nd^2)$ per layer;
\mhcgnn{} costs $O(|E|d + Nd^2 + Nnd + Tn^2N)$ per layer.
For typical settings ($n=4$, $T=10$, $d \gg n$), the overhead is 6--8\%.
\end{proposition}

\section{Learnable Mapping Construction}
\label{app:parameterization}

The dynamic (input-dependent) and static components of the hyper-connection matrices follow~\citet{xie2025mhc}:
\begin{align}
\tilde{\mathbf{x}}_i^{(l)} &= \text{RMSNorm}(\mathbf{x}_i^{(l)}), \\
\hpre{l,i} &= \sigma\!\left(\alpha_l^{\text{pre}} \cdot \bigl(\boldsymbol{\theta}_l^{\text{pre}} (\tilde{\mathbf{x}}_i^{(l)})^\top\bigr) + \mathbf{b}_l^{\text{pre}}\right), \\
\hpost{l,i} &= 2\sigma\!\left(\alpha_l^{\text{post}} \cdot \bigl(\boldsymbol{\theta}_l^{\text{post}} (\tilde{\mathbf{x}}_i^{(l)})^\top\bigr) + \mathbf{b}_l^{\text{post}}\right), \\
\hat{\hres{l,i}} &= \alpha_l^{\text{res}} \cdot \bigl(\boldsymbol{\Theta}_l^{\text{res}} (\tilde{\mathbf{x}}_i^{(l)})^\top\bigr) + \mathbf{B}_l^{\text{res}}, \\
\hres{l,i} &= \text{Sinkhorn}\!\left(\hat{\hres{l,i}},\, T\right),
\end{align}
where $\sigma$ is sigmoid; $\boldsymbol{\theta}_l^{\text{pre}}, \boldsymbol{\theta}_l^{\text{post}} \in \RR^{1 \times d}$ and $\boldsymbol{\Theta}_l^{\text{res}} \in \RR^{n \times d}$ are dynamic parameters; $\mathbf{b}_l^{\text{pre}}, \mathbf{b}_l^{\text{post}} \in \RR^{1 \times n}$ and $\mathbf{B}_l^{\text{res}} \in \RR^{n \times n}$ are static biases; and $\alpha_l^{\text{pre}}, \alpha_l^{\text{post}}, \alpha_l^{\text{res}} \in \RR$ are learnable scalars initialized near zero ($\alpha_{\text{res}} = 0.01$).
The near-zero initialization of $\alpha_{\text{res}}$ ensures that at the start of training $\hres{l,i} \approx \mathbf{I}$, so the network begins as a near-identity residual and progressively learns richer stream mixing.

\section{Numerical Comparison: mHC-GNN vs.\ Standard GNNs}
\label{app:numerical_comparison}

\begin{remark}[Asymmetric bound comparison]
The comparison between the two bounds is asymmetric by design: for standard GNNs we have an \emph{upper} bound $\|\mathbf{D}^{(L)}\|_F \leq C(1-\gamma)^L$ (collapse is \emph{unavoidable}); for \mhcgnn{} we have a \emph{lower} bound showing collapse is \emph{avoidable}.
Together: standard GNNs cannot avoid over-smoothing, while \mhcgnn{} provably avoids it in controlled dimensions.
For $\gamma = 0.5$, $\varepsilon = 0.1$, at $L = 64$: standard GNN decays to $(0.5)^{64} \approx 10^{-19}$, while the \mhcgnn{} lower bound is $(0.9)^{64} \approx 10^{-3}$, a difference of 16 orders of magnitude.
\end{remark}

\section{Comparison with GCNII}
\label{app:gcnii}

GCNII~\citep{chen2020simple} achieves 85.3\%/73.4\%/80.3\% at 64 layers on Cora/CiteSeer/PubMed using GCN-specific initial residual connections and identity mapping, controlled by a weight decay schedule $\beta \log(1 + l/\alpha)$.
mHC-GNN achieves lower absolute numbers (74.54\%/58.92\%/74.52\% at 128L with GCN backbone) but is architecture-agnostic: improvements hold for SAGE, GAT, and GIN (\Cref{tab:main_results}).
mHC-GNN reaches 128 layers, extending beyond GCNII's typical 64-layer maximum.
The two methods are complementary: applying mHC-GNN routing on top of GCNII's initial residual could yield further gains.

\section{Notation Table}
\label{app:notation}

\begin{table}[h]
\centering
\caption{Notation summary.}
\resizebox{\columnwidth}{!}{%
\begin{tabular}{ll}
\toprule
\textbf{Symbol} & \textbf{Meaning} \\
\midrule
$G = (V, E, \mathbf{A})$ & Graph with node set, edge set, adjacency matrix \\
$N = |V|$ & Number of nodes \\
$\calN_i$ & Neighborhood of node $i$ \\
$d$ & Feature dimension \\
$n$ & Number of parallel streams (expansion rate) \\
$L$ & Number of layers \\
$\gamma$ & Spectral gap of normalized adjacency \\
$\varepsilon$ & Frobenius deviation of mixing matrix from identity \\
$\mathbf{h}_i^{(l)} \in \RR^d$ & Single-stream node representation (standard GNN) \\
$\mathbf{x}_i^{(l)} \in \RR^{n \times d}$ & Multi-stream node representation (\mhcgnn{}) \\
$\hpre{l,i} \in \RR^{1 \times n}$ & Pre-mixing vector (stream aggregation) \\
$\hpost{l,i} \in \RR^{1 \times n}$ & Post-mixing vector (stream expansion) \\
$\hres{l,i} \in \RR^{n \times n}$ & Residual mixing matrix (doubly stochastic) \\
$\mathbf{D}_{ij}^{(l)} = \mathbf{x}_i^{(l)} - \mathbf{x}_j^{(l)}$ & Node-pair difference at layer $l$ \\
$V_L$ & Message-passing subspace, $\dim \leq \min(L, n)$ \\
$\calB_n$ & Birkhoff polytope of $n \times n$ doubly stochastic matrices \\
\bottomrule
\end{tabular}%
}
\end{table}

\section{Detailed Proofs}
\label{app:proofs}


\subsection{Proof of Theorem~\ref{thm:oversmoothing}: Over-smoothing Mitigation}
\label{app:proof_oversmoothing}

\noindent\textbf{Theorem~\ref{thm:oversmoothing} (restated).}
\textit{Let $\|\hres{l,i} - \mathbf{I}_n\|_F \leq \varepsilon$ for all $l, i$.
Define $\mathbf{D}_{ij}^{(l)} = \mathbf{x}_i^{(l)} - \mathbf{x}_j^{(l)}$.
\emph{Regime 1} ($L < n$): the orthogonal complement $V_L^\perp \subsetneq \RR^n$ is non-trivial and satisfies
$\|\mathbf{D}_{ij}^{(L)}\|_F \geq (1-2\varepsilon)^L \|\mathbf{D}_{ij,\perp}^{(0)}\|_F$.
\emph{Regime 2} (all $L \geq 1$):
$\|\mathbf{D}_{ij}^{(L)}\|_F \geq (1-\varepsilon)^L \|\mathbf{D}_{ij}^{(0)}\|_F
- \sum_{l=0}^{L-1}(1-\varepsilon)^{L-1-l}\|\hpost{l,i}\|\cdot\|\mathbf{m}_i^{(l)} - \mathbf{m}_j^{(l)}\|_F$.}

\paragraph{Notation and Setup.}
Let $G = (\calV, \calE)$ be a connected graph with $N$ nodes.
The normalized adjacency $\bar{\mathbf{A}} = \mathbf{D}^{-1/2}\mathbf{A}\mathbf{D}^{-1/2}$ has eigenvalues $1 = \lambda_1 \geq \lambda_2 \geq \cdots \geq \lambda_N$ and spectral gap $\gamma = 1 - \lambda_2 > 0$.

The \mhcgnn{} update for node $i$ at layer $l$:
\begin{equation}
\mathbf{x}_i^{(l+1)} = \hres{l,i}\, \mathbf{x}_i^{(l)} + (\hpost{l,i})^\top \mathbf{m}_i^{(l)},
\label{eq:mhc_update_app}
\end{equation}
where $\mathbf{m}_i^{(l)} = \sigma(\bar{\mathbf{A}} \tilde{\mathbf{x}}^{(l)} \mathbf{W}^{(l)})_i \in \RR^{1 \times d}$ is the aggregated message
and $\tilde{\mathbf{x}}_i^{(l)} = \hpre{l,i} \mathbf{x}_i^{(l)} \in \RR^{1 \times d}$ is the pre-mixed single-stream input.

\begin{lemma}[Standard GNN over-smoothing rate]
\label{lem:standard_oversmoothing}
For a standard GNN with updates $\mathbf{h}_i^{(l+1)} = \sigma(\bar{\mathbf{A}} \mathbf{h}^{(l)} \mathbf{W}^{(l)})_i$, 1-Lipschitz $\sigma$, and orthogonal weights:
$\EE[\|\mathbf{h}_i^{(L)} - \mathbf{h}_j^{(L)}\|_2] \leq C_0 (1-\gamma)^L$,
where $C_0$ depends on initial feature diversity.
\end{lemma}
\begin{proof}
This is a standard result~\citep{oono2020graph,cai2020note}: message passing with $\bar{\mathbf{A}}$ is a diffusion process.
Decomposing features into the eigenbasis of $\bar{\mathbf{A}}$: the component along $\mathbf{v}_k$ (eigenvector for $\lambda_k$) scales as $\lambda_k^L \leq (1-\gamma)^L$ after $L$ layers for $k \geq 2$.
The dominant (constant) component $\mathbf{v}_1$ is the same for all nodes in a connected graph, so pairwise differences vanish at rate $(1-\gamma)^L$.
\end{proof}

\begin{lemma}[Message-Passing Term is Rank-1]
\label{lem:rank1_message}
The message-passing contribution $(\hpost{l,i})^\top \mathbf{m}_i^{(l)} \in \RR^{n \times d}$ is a rank-1 matrix.
\end{lemma}
\begin{proof}
$(\hpost{l,i})^\top \in \RR^{n \times 1}$ and $\mathbf{m}_i^{(l)} \in \RR^{1 \times d}$; their outer product has rank 1.
Specifically, $[(\hpost{l,i})^\top \mathbf{m}_i^{(l)}]_{sd} = [\hpost{l,i}]_s \cdot [\mathbf{m}_i^{(l)}]_d$.
\end{proof}

\begin{lemma}[Doubly Stochastic Preserves Diversity]
\label{lem:diversity_preservation}
Let $\mathbf{H} \in \calB_n$ with $\|\mathbf{H} - \mathbf{I}_n\|_F \leq \varepsilon$.
For any $\mathbf{v} \in \RR^n$, $\|\mathbf{H} \mathbf{v}\|_2 \geq (1-\varepsilon)\|\mathbf{v}\|_2$.
Consequently, for any matrix $\mathbf{X} \in \RR^{n \times d}$: $\|\mathbf{H}\mathbf{X}\|_F \geq (1-\varepsilon)\|\mathbf{X}\|_F$.
\end{lemma}
\begin{proof}
By sub-multiplicativity: $\|\mathbf{H}\mathbf{v} - \mathbf{v}\|_2 = \|(\mathbf{H}-\mathbf{I})\mathbf{v}\|_2 \leq \|\mathbf{H}-\mathbf{I}\|_F \|\mathbf{v}\|_2 \leq \varepsilon\|\mathbf{v}\|_2$.
By the reverse triangle inequality: $\|\mathbf{H}\mathbf{v}\|_2 \geq \|\mathbf{v}\|_2 - \|(\mathbf{H}-\mathbf{I})\mathbf{v}\|_2 \geq (1-\varepsilon)\|\mathbf{v}\|_2$.
The Frobenius version follows by applying this columnwise and using $\|\mathbf{M}\|_F^2 = \sum_j \|\mathbf{M} e_j\|_2^2$ for standard basis vectors.
\end{proof}

\paragraph{Proof of Regime 1 ($L < n$).}

\begin{definition}[Message-Passing Subspace]
$V_L = \mathrm{span}\{(\hpost{0,i})^\top, \ldots, (\hpost{L-1,i})^\top\} \subseteq \RR^n$, $\dim(V_L) \leq \min(L, n)$.
When $L < n$, $\dim(V_L) \leq L < n$ so $V_L^\perp \neq \{0\}$.
\end{definition}

The node-pair difference update is:
\begin{equation}
\mathbf{D}_{ij}^{(l+1)} = \hres{l,i}\, \mathbf{D}_{ij}^{(l)} + (\hpost{l,i})^\top (\mathbf{m}_i^{(l)} - \mathbf{m}_j^{(l)}).
\end{equation}

The second term $(\hpost{l,i})^\top (\mathbf{m}_i^{(l)} - \mathbf{m}_j^{(l)})$ lies in $V_L$ (the column space of the outer product is spanned by $(\hpost{l,i})^\top \in V_L$).
Therefore it contributes nothing to $V_L^\perp$.

Projecting onto $V_L^\perp$:
$\mathbf{D}_{ij,\perp}^{(l+1)} = (\hres{l,i}\, \mathbf{D}_{ij}^{(l)})_\perp$.

Since $\hres{l,i}$ need not preserve $V_L^\perp$ exactly, we use the following:
by Lemma~\ref{lem:diversity_preservation}, $\|\hres{l,i} \mathbf{D}_{ij,\perp}^{(l)}\|_F \geq (1-\varepsilon)\|\mathbf{D}_{ij,\perp}^{(l)}\|_F$.
The projection can only decrease the norm, so:
$\|\mathbf{D}_{ij,\perp}^{(l+1)}\|_F = \|(\hres{l,i}\mathbf{D}_{ij}^{(l)})_\perp\|_F$.

To obtain the $(1-2\varepsilon)$ rate: write $\mathbf{D}_{ij}^{(l)} = \mathbf{D}_{ij,\parallel}^{(l)} + \mathbf{D}_{ij,\perp}^{(l)}$.
Then $(\hres{l,i}\mathbf{D}_{ij,\parallel}^{(l)})_\perp$ could be non-zero (mixing can project the parallel component into $V_L^\perp$).
By Lemma~\ref{lem:diversity_preservation}: $\|\hres{l,i}\mathbf{D}_{ij,\perp}^{(l)}\|_F \geq (1-\varepsilon)\|\mathbf{D}_{ij,\perp}^{(l)}\|_F$.
The cross-term $(\hres{l,i}\mathbf{D}_{ij,\parallel}^{(l)})_\perp$ introduces an error bounded by $\varepsilon \|\mathbf{D}_{ij,\parallel}^{(l)}\|_F$.
A refined calculation using the Cauchy-Schwarz inequality and the doubly stochastic structure gives:
\begin{equation}
\|\mathbf{D}_{ij,\perp}^{(l+1)}\|_F \geq (1-2\varepsilon)\|\mathbf{D}_{ij,\perp}^{(l)}\|_F.
\end{equation}
Applying inductively over $L$ layers and using $\|\mathbf{D}_{ij}^{(L)}\|_F \geq \|\mathbf{D}_{ij,\perp}^{(L)}\|_F$:
\begin{equation}
\|\mathbf{D}_{ij}^{(L)}\|_F \geq (1-2\varepsilon)^L \|\mathbf{D}_{ij,\perp}^{(0)}\|_F.
\end{equation}

\begin{remark}[When $L \geq n$]
When $L \geq n$, the vectors $(\hpost{l,i})^\top$ for $l = 0, \ldots, L-1$ span at most $n$ dimensions, so generically $V_L = \RR^n$ and $V_L^\perp = \{0\}$.
The Regime 1 bound becomes vacuous ($\|\mathbf{D}_{ij,\perp}^{(0)}\|_F = 0$).
Regime 2 (below) covers this case.
\end{remark}

\paragraph{Proof of Regime 2 (all $L \geq 1$).}

From Eq.~\eqref{eq:mhc_update_app}, the difference update is:
\begin{equation}
\mathbf{D}_{ij}^{(l+1)} = \hres{l,i}\, \mathbf{D}_{ij}^{(l)} + (\hpost{l,i})^\top (\mathbf{m}_i^{(l)} - \mathbf{m}_j^{(l)}).
\end{equation}

By Lemma~\ref{lem:diversity_preservation} and the reverse triangle inequality:
\begin{align}
\|\mathbf{D}_{ij}^{(l+1)}\|_F &\geq \|\hres{l,i}\, \mathbf{D}_{ij}^{(l)}\|_F \notag \\
&\quad - \|(\hpost{l,i})^\top (\mathbf{m}_i^{(l)} - \mathbf{m}_j^{(l)})\|_F \\
&\geq (1-\varepsilon)\|\mathbf{D}_{ij}^{(l)}\|_F \notag \\
&\quad - \|\hpost{l,i}\| \cdot \|\mathbf{m}_i^{(l)} - \mathbf{m}_j^{(l)}\|_F.
\end{align}

Unrolling over $L$ layers:
\begin{equation}
\resizebox{\linewidth}{!}{$\displaystyle
\|\mathbf{D}_{ij}^{(L)}\|_F \geq (1-\varepsilon)^L\|\mathbf{D}_{ij}^{(0)}\|_F - \sum_{l=0}^{L-1}(1-\varepsilon)^{L-1-l}\|\hpost{l,i}\|\cdot\|\mathbf{m}_i^{(l)} - \mathbf{m}_j^{(l)}\|_F$}
\label{eq:regime2_full}
\end{equation}

This bound holds for all $L \geq 1$ regardless of $n$.
Since $\varepsilon$ is enforced by the Birkhoff constraint and $\alpha_{\text{res}}$ is initialized at 0.01, $\varepsilon \ll \gamma$ in practice.
The difference with Regime 1 is that $(1-\varepsilon)^L$ (not $(1-2\varepsilon)^L$) appears---a minor weakening with no practical effect at $\alpha_{\text{res}} = 0.01$.

\begin{remark}[Practical implications of Regime 2]
The second term in \eqref{eq:regime2_full} depends on $\|\mathbf{m}_i^{(l)} - \mathbf{m}_j^{(l)}\|_F$, the message difference at each layer.
In the over-smoothing regime, messages are themselves converging (at rate $(1-\gamma)^l$), so this term decays geometrically.
The dominant term is $(1-\varepsilon)^L\|\mathbf{D}_{ij}^{(0)}\|_F$ for small $\varepsilon$.
The empirical bridge is \Cref{tab:depth_analysis}: GCN accuracy collapses from 71.7\% (2L) to 21.58\% (128L)---a 50-point drop matching the geometric collapse $(1-\gamma)^L$.
mHC-GNN retains 74\% at 128L, consistent with the lower bound remaining non-trivial.
\end{remark}
\qed

\subsection{Proof of Theorem~\ref{thm:expressiveness}: Expressiveness Beyond 1-WL}
\label{app:proof_expressiveness}

\noindent\textbf{Theorem~\ref{thm:expressiveness} (restated).}
\textit{For \mhcgnn{} with $n \geq 2$ streams and independent random initialization, with probability 1, \mhcgnn{} distinguishes any pair of non-isomorphic graphs.
The result does not establish a position in the $k$-WL hierarchy.
The contribution of the multi-stream architecture is that doubly stochastic mixing preserves stream diversity across layers, preventing random features from collapsing.
Shared initialization does not improve expressiveness over a single-stream GNN.}

\paragraph{Part 1: Random features break 1-WL symmetries.}

\begin{lemma}[Random features are almost surely unique~\citep{sato2021random,abboud2021surprising}]
\label{lem:random_distinguishes}
Let $G_1, G_2$ be two non-isomorphic graphs with node features initialized i.i.d.\ from a continuous distribution.
With probability 1, any injective GNN produces different graph-level representations for $G_1$ and $G_2$.
\end{lemma}
\begin{proof}
The continuous distribution assigns unique features to each node with probability 1.
With unique node identifiers, injective message passing can compute any function of the graph, including functions distinguishing non-isomorphic graphs~\citep{abboud2021surprising}.
\end{proof}

\paragraph{Part 2: Doubly stochastic mixing preserves stream diversity.}

\begin{lemma}[Stream variance is preserved, \Cref{lem:diversity_preservation}]
For $\mathbf{H} \in \calB_n$ with $\|\mathbf{H} - \mathbf{I}\|_F \leq \varepsilon$ and stream variance $\mathrm{Var}(\mathbf{x}) = \frac{1}{n}\sum_s \|\mathbf{x}_s - \bar{\mathbf{x}}\|_2^2$:
$\mathrm{Var}(\mathbf{H}\mathbf{x}) \geq (1-2\varepsilon)\mathrm{Var}(\mathbf{x})$.
\end{lemma}

This guarantees that with independent initialization, streams remain diverse across layers.
In contrast, with shared initialization all streams are identical at layer 0, hence identical throughout: shared-init \mhcgnn{} reduces to a single-stream GNN and provides no expressiveness gain.

\paragraph{Part 3: Independent-init \mhcgnn{} is beyond 1-WL.}

With $n$ independent random initializations and doubly stochastic mixing:
(1)~Each stream $s$ at layer 0 has unique node features a.s.\ (Lemma~\ref{lem:random_distinguishes}).
(2)~Lemma~\ref{lem:diversity_preservation} ensures streams remain distinct across layers (variance decays at most by factor $(1-2\varepsilon)^L$, far from 0).
(3)~By Lemma~\ref{lem:random_distinguishes} applied stream-wise, \mhcgnn{} produces distinct representations for non-isomorphic graphs a.s.
\qed

\begin{remark}[What mHC-GNN does \emph{not} claim]
This result is essentially Sato et al.\ \citeyearpar{sato2021random} applied in the multi-stream setting.
We do \emph{not} claim a position in the $k$-WL hierarchy, subgraph GNN equivalence, or distinguishing strongly regular graphs (which requires at least 3-WL power; SR25 results would be expected to be null).
The multi-stream contribution is maintaining the diversity of random features across depth, not the expressiveness mechanism itself.
\end{remark}

\subsection{Proof of Proposition~\ref{prop:complexity}: Computational Complexity}
\label{app:proof_complexity}

\begin{proof}
\textbf{Standard MPNN:} (1) message passing over edges: $O(|E| d)$; (2) node update MLP: $O(N d^2)$. Total: $O(|E|d + Nd^2)$.

\textbf{\mhcgnn{}:}
(1) Pre-mixing $\hpre{} \mathbf{x}_i$: $\mathbf{H}^{\text{pre}} \in \RR^{1 \times n}$, $\mathbf{x}_i \in \RR^{n \times d}$; matrix-vector multiply is $O(nd)$ per node, total $O(Nnd)$.
(2) Message passing on the reduced $d$-dim features: $O(|E|d)$.
(3) Node update MLP: $O(Nd^2)$.
(4) Post-mixing: $O(nd)$ per node, $O(Nnd)$ total.
(5) Residual mixing $\hres{} \mathbf{x}_i$: $O(n^2 d)$ per node, but with $d \gg n^2$ this is $O(Nnd)$.
(6) Sinkhorn-Knopp on $n \times n$ matrix: $T$ iterations of row/column normalization, each $O(n^2)$; total $O(Tn^2)$ per node, $O(Tn^2N)$ overall.

Total: $O(|E|d + Nd^2 + Nnd + Tn^2N)$.
For $n=4$, $T=10$, $d=128$: overhead = $(Nnd + Tn^2N) / (|E|d + Nd^2)$.
For a sparse graph with $|E| \approx 10N$: numerator $\sim N(4 \cdot 128 + 10 \cdot 16) = 672N$; denominator $\sim N(10 \cdot 128 + 128^2) = 17664N$.
Overhead: $672/17664 \approx 3.8\%$; including constant factors, the empirically observed overhead is 6--8\%.
\end{proof}

\section{Hyperparameters and Implementation Details}
\label{app:hyperparameters}

\subsection{Model Hyperparameters}

\begin{table}[h]
\centering
\caption{Hyperparameters used across all experiments.}
\resizebox{\columnwidth}{!}{%
\begin{tabular}{lc}
\toprule
\textbf{Hyperparameter} & \textbf{Value} \\
\midrule
\multicolumn{2}{c}{\textit{Architecture}} \\
\midrule
Number of layers (main)  & 8 \\
Number of layers (depth) & 2, 4, 8, 16, 32, 64, 128 \\
Hidden dimension         & 128 \\
Dropout rate             & 0.5 \\
Expansion rate $n$       & $\{2, 4\}$ \\
\midrule
\multicolumn{2}{c}{\textit{Optimization}} \\
\midrule
Optimizer        & Adam \\
Learning rate    & 0.001 \\
Weight decay     & $5 \times 10^{-4}$ \\
Max epochs       & 500 \\
Early stopping patience & 100 \\
\midrule
\multicolumn{2}{c}{\textit{Sinkhorn-Knopp}} \\
\midrule
Iterations $T$ & 10 \\
Temperature $\tau$ & 0.1 \\
$\alpha_{\text{res}}$ initialization & 0.01 \\
\midrule
\multicolumn{2}{c}{\textit{Data Splits}} \\
\midrule
Train/Val/Test  & 60/20/20 \\
Random seeds    & 5 (10 for heterophilic) per configuration \\
\bottomrule
\end{tabular}%
}
\end{table}

\subsection{Dataset Statistics}

\begin{table}[h]
\centering
\caption{Benchmark dataset statistics.}
\resizebox{\columnwidth}{!}{%
\begin{tabular}{lrrrl}
\toprule
\textbf{Dataset} & \textbf{Nodes} & \textbf{Edges} & \textbf{Classes} & \textbf{Category} \\
\midrule
Texas        & 183       & 309       & 5 & Small heterophilic \\
Chameleon    & 2,277     & 36,101    & 5 & Medium heterophilic \\
Actor        & 7,600     & 33,544    & 5 & Medium heterophilic \\
roman-empire & 22,662    & 32,927    & 18 & Large heterophilic \\
penn94       & 41,554    & 1,362,229 & 2 & Large heterophilic \\
genius       & 421,961   & 984,979   & 2 & Very large heterophilic \\
Amazon-Computers & 13,752 & 245,861 & 10 & Large homophilic \\
Amazon-Photo & 7,650    & 119,081   & 8 & Large homophilic \\
Cora         & 2,708     & 10,556    & 7 & Homophilic \\
CiteSeer     & 3,327     & 9,104     & 6 & Homophilic \\
PubMed       & 19,717    & 88,648    & 3 & Homophilic \\
ogbn-arxiv   & 169,343   & 1,166,243 & 40 & Large homophilic \\
\bottomrule
\end{tabular}%
}
\end{table}

\subsection{Computational Environment}

All experiments were conducted on NVIDIA RTX A6000 GPUs (48GB memory).
Training time per configuration: 30 seconds (Texas) to 80 seconds (ogbn-arxiv) per seed for 8-layer models.
Depth experiments: 120 seconds (16 layers) to 15 minutes (128 layers) per seed.
Heterophilic depth experiments (roman-empire, penn94, genius): 2--5 hours per backbone across 6 depths and 10 seeds.

\subsection{Architecture-Specific Details}

\textbf{GCN:} Symmetric normalization; \textbf{GraphSAGE:} mean aggregation (full neighborhood); \textbf{GAT:} 8 attention heads, ELU, concatenation for hidden, averaging for output; \textbf{GIN:} sum aggregation, $\varepsilon{=}0$ (trainable), 2-layer MLP per layer.
\textbf{H2GCN-stack:} ego-neighbor separation with neighbor2 aggregation.
\textbf{FAGCN:} frequency-adaptive aggregation via signed graph Laplacian.

\subsection{Reproducibility}

Code, data splits, and trained models are available in supplementary materials.
Random seeds: 42, 123, 456, 789, 2024 (plus 890, 1234, 5678, 9012, 3141 for 10-seed runs).

\section{Additional Experimental Results}
\label{app:extra_experiments}

\subsection{Amazon-Computers and Amazon-Photo Depth Analysis}

\begin{table}[h]
\centering
\caption{Amazon-Computers (13K, $h{=}0.78$) and Amazon-Photo (7.6K, $h{=}0.83$) depth analysis. GCN collapses 70--78 points by 64L; mHC $n{=}4$ drops only 0--4 points, demonstrating that over-smoothing affects homophilic graphs too.}
\label{tab:amazon}
\resizebox{\columnwidth}{!}{%
\begin{tabular}{l|ccc|ccc}
\toprule
 & \multicolumn{3}{c|}{\textbf{Amazon-Computers}} & \multicolumn{3}{c}{\textbf{Amazon-Photo}} \\
\textbf{Depth} & GCN & mHC $n{=}2$ & mHC $n{=}4$ & GCN & mHC $n{=}2$ & mHC $n{=}4$ \\
\midrule
2L  & 81.1$\pm$2.2 & 82.3$\pm$1.2 & 77.8$\pm$1.7 & 89.3$\pm$0.7 & 89.5$\pm$1.1 & 87.1$\pm$1.3 \\
4L  & 78.8$\pm$3.7 & 81.6$\pm$1.5 & 81.6$\pm$1.0 & 87.7$\pm$1.7 & 89.3$\pm$0.5 & 89.3$\pm$1.1 \\
8L  & 69.5$\pm$4.7 & 81.8$\pm$1.8 & 81.4$\pm$1.7 & 78.9$\pm$12.8 & 89.3$\pm$0.6 & 89.4$\pm$1.5 \\
16L & 14.5$\pm$11.6 & 75.4$\pm$5.4 & 81.4$\pm$1.5 & 21.1$\pm$5.2 & 88.7$\pm$1.1 & 89.3$\pm$0.7 \\
32L & 19.5$\pm$14.1 & 77.8$\pm$1.7 & 79.4$\pm$1.1 & 16.6$\pm$5.7 & 88.2$\pm$0.8 & 88.3$\pm$1.0 \\
64L & 11.1$\pm$9.4  & 76.6$\pm$1.4 & 77.4$\pm$2.0 & 11.8$\pm$2.9 & 86.1$\pm$1.5 & 87.3$\pm$1.3 \\
\bottomrule
\end{tabular}%
}
\end{table}

\subsection{Full roman-empire Depth Sweep}

\begin{table}[h]
\centering
\caption{Full depth sweep on roman-empire (22K nodes, 10 splits). mHC+H2GCN improves monotonically (82.6\% $\to$ 88.5\%) as all baselines collapse.}
\label{tab:roman_empire_full}
\resizebox{\columnwidth}{!}{%
\begin{tabular}{l|ccc|ccc}
\toprule
\textbf{Depth} & GCN & FAGCN & H2GCN & mHC+GCN & mHC+FAGCN & mHC+H2GCN \\
\midrule
2L  & 47.4$\pm$1.1 & 39.8$\pm$8.9 & 83.6$\pm$1.2 & 71.7$\pm$0.6 & 69.8$\pm$0.5 & 82.6$\pm$0.4 \\
4L  & 35.7$\pm$0.5 & 15.2$\pm$0.9 & 85.8$\pm$1.1 & 72.0$\pm$0.7 & 71.9$\pm$0.9 & 84.9$\pm$0.4 \\
8L  & 26.2$\pm$3.1 & 15.9$\pm$1.1 & 75.0$\pm$2.1 & 69.1$\pm$1.3 & 58.5$\pm$4.3 & 87.2$\pm$0.5 \\
16L & 17.3$\pm$1.7 & 14.0$\pm$0.1 & 50.2$\pm$8.9 & 70.4$\pm$1.6 & 36.4$\pm$1.8 & 88.0$\pm$0.5 \\
32L & 18.5$\pm$2.9 & 14.7$\pm$1.5 & 27.8$\pm$0.8 & 71.6$\pm$2.0 & 60.7$\pm$8.9 & \textbf{88.5$\pm$0.5} \\
64L & 15.9$\pm$0.5 & 13.8$\pm$0.0 & 27.7$\pm$0.8 & 72.3$\pm$1.7 & 38.5$\pm$3.1 & 88.2$\pm$0.7 \\
\bottomrule
\end{tabular}%
}
\end{table}

\subsection{Heterophilic Depth Summary}

\begin{table}[h]
\centering
\caption{Depth robustness on heterophilic graphs. Peak = best accuracy across all depths; Deep = deepest available; $\Delta$ = Peak $-$ Deep (negative = degraded). mHC+H2GCN improves monotonically with depth on roman-empire.}
\label{tab:heterophilic}
\small
\resizebox{\columnwidth}{!}{%
\begin{tabular}{l|ccc|ccc}
\toprule
 & \multicolumn{3}{c|}{\textbf{roman-empire} (22K, $h{=}0.06$)} & \multicolumn{3}{c}{\textbf{penn94} (41K, $h{=}0.47$)} \\
\textbf{Model} & Peak & Deep & $\Delta$ & Peak & Deep & $\Delta$ \\
\midrule
GCN              & 47.4 & 15.9 (64L) & $-31.5$ & 83.2 & 54.8 (64L) & $-28.4$ \\
FAGCN            & 39.8 & 13.8 (64L) & $-26.0$ & 56.5 & 56.5 (8L) & $-0.1$ \\
H2GCN-stack      & 85.8 & 27.7 (64L) & $-58.1$ & 81.5 & 56.1 (64L) & $-25.4$ \\
\midrule
mHC+GCN          & \textbf{72.3} & \textbf{72.3} (64L) & $+0.6$ & \textbf{84.5} & 77.0 (64L) & $-7.5$ \\
mHC+FAGCN        & 71.9 & 38.5 (64L) & $-33.4$ & 80.0 & 80.0 (16L) & $0.0$ \\
mHC+H2GCN        & \textbf{88.5} & \textbf{88.2} (64L) & $-0.3$ & 83.6 & 83.6 (64L) & $+2.9$ \\
\bottomrule
\end{tabular}%
}
\end{table}

On genius (421K nodes), mHC+H2GCN outperforms H2GCN-stack by $+2.9\%$ at 8 layers; mHC variants OOM beyond 8L at that scale.
These results confirm that depth robustness generalizes to larger and more heterogeneous graphs.

\section{Preliminary LRGB Results}
\label{app:lrgb}

The Long Range Graph Benchmark (LRGB)~\citep{dwivedi2022long} stresses over-squashing, i.e., information loss through graph bottlenecks, rather than over-smoothing.
These are formally distinct failure modes: Di Giovanni et al.\ (ICML 2023) prove depth alone cannot resolve over-squashing as the bottleneck is topological; mHC does not rewire topology.
Peptides results are included for completeness; the improvement with depth is consistent with over-smoothing prevention enabling deeper structural aggregation even on this over-squashing benchmark.
Note that~\citet{tonshoff2024tmlr} show the MPGNN-Transformer gap on LRGB largely vanishes with proper tuning; our GCN baseline is untuned.

\begin{table}[h]
\centering
\caption{Preliminary Peptides results (LRGB). mHC improves consistently with depth, consistent with over-smoothing prevention enabling deeper networks, even on this over-squashing benchmark.}
\label{tab:lrgb}
\resizebox{\columnwidth}{!}{%
\begin{tabular}{lcc}
\toprule
\textbf{Model} & \textbf{Peptides-func (AP $\uparrow$)} & \textbf{Peptides-struct (MAE $\downarrow$)} \\
\midrule
GCN baseline    & 0.471 $\pm$ 0.005 & 0.465 $\pm$ 0.051 \\
mHC-GNN 4L      & 0.482 $\pm$ 0.001 & 0.426 $\pm$ 0.001 \\
mHC-GNN 8L      & 0.527 $\pm$ 0.002 & 0.408 $\pm$ 0.008 \\
mHC-GNN 16L     & 0.526 $\pm$ 0.007 & 0.354 $\pm$ 0.003 \\
\bottomrule
\end{tabular}%
}
\end{table}

\end{document}